\documentclass[lettersize,journal]{IEEEtran}[10pt]
\usepackage{amsmath,amsfonts}
\usepackage{algorithmic}
\usepackage{algorithm}
\usepackage{array}
\usepackage{amsthm}
\usepackage[caption=false,font=normalsize,labelfont=sf,textfont=sf]{subfig}
\usepackage{textcomp}
\usepackage{stfloats}
\usepackage{xcolor}
\usepackage{url}
\newcommand{\Kappa}{\mathcal{K}}
\usepackage{verbatim}
\usepackage{graphicx}
\usepackage{cite}
\usepackage{caption}
\newtheorem{theorem}{Theorem}

\usepackage{mathrsfs}
\usepackage[
    colorlinks=true,    
    linkcolor=blue,     
    citecolor=blue,     
    urlcolor=blue       
]{hyperref}
\begin{document}

\title{Quasi-Periodic Gaussian Process Predictive Iterative Learning Control}
\author{Unnati Nigam\thanks{Unnati Nigam is a Ph.D. student at IITB-Monash Research Academy, IIT Bombay, Mumbai, India.}, Radhendushka Srivastava\thanks{Radhendushka Srivastava is with the Department of Mathematics, IIT Bombay, Mumbai, India.}, Faezeh Marzbanrad, Michael Burke\thanks{Faezeh Marzbanrad and Michael Burke are with the Department of Electrical and Computer Systems Engineering, Monash University, Clayton, Melbourne, Australia.}}

\maketitle

\begin{abstract} 
Repetitive motion tasks are common in robotics, but performance can degrade over time due to environmental changes and robot wear and tear. Iterative learning control (ILC) improves performance by using information from previous iterations to compensate for expected errors in future iterations. This work incorporates the use of Quasi-Periodic Gaussian Processes (QPGPs) into a predictive ILC framework to model and forecast disturbances and drift across iterations. Using a recent structural equation formulation of QPGPs, the proposed approach enables efficient inference with complexity $\mathcal{O}(p^3)$ instead of $\mathcal{O}(i^2p^3)$, where $p$ denotes the number of points within an iteration and $i$ represents the total number of iterations, specially for larger $i$. This formulation also enables parameter estimation without loss of information, making continual GP learning computationally feasible within the control loop. By predicting next-iteration error profiles rather than relying only on past errors, the controller achieves faster convergence and maintains this under time-varying disturbances. We benchmark the method against both standard ILC and conventional Gaussian Process (GP)-based predictive ILC on three tasks, autonomous vehicle trajectory tracking, a three-link robotic manipulator, and a real-world Stretch robot experiment. Across all cases, the proposed approach converges faster and remains robust under injected and natural disturbances while reducing computational cost. This highlights its practicality across a range of repetitive dynamical systems. 
\end{abstract}

\begin{IEEEkeywords}
Model Learning for Control, Probability and Statistical Methods, Machine Learning for Robot Control. 
\end{IEEEkeywords}

\section{Introduction}
ILC is a powerful technique for improving the performance of systems executing repetitive tasks~\cite{survey_ilc}. By leveraging information from previous iterations, ILC aims to reduce tracking errors over successive repetitions, making it particularly suitable for robotic manipulators, autonomous vehicles, and other cyclic dynamical systems \cite{ilc_formula, ilc_book}. ILC has been widely used to improve robotic trajectory tracking and motion control \cite{robot_ilc1,robot_ilc2, robot_app}, as well as in autonomous driving applications \cite{car_app} and industrial manipulators \cite{manipulator_app}.
However, standard ILC approaches only update control inputs reactively using past errors, despite the fact that error patterns often repeat or evolve predictably across iterations. This may lead to slow convergence, especially under disturbances or when errors exhibit quasi-periodic behaviour.

While Quasi-Periodic Gaussian Processes (QPGP) have recently been proposed for efficient modeling of pseudo-periodic signals, they have not been integrated into iterative learning control in a way that supports continual, iteration-by-iteration learning with bounded computation. This work shows how a structural equation form of QPGPs can be embedded directly into predictive ILC, enabling next-iteration error prediction using only the most recent trial, while remaining conditioned on all past data.

We demonstrate the effectiveness of the proposed QPGP-based predictive ILC through three case studies: autonomous vehicle trajectory tracking, a 3-link planar manipulator following a circular trajectory, and a real-world Hello Stretch robot executing a Lissajous curve tracking task. 
These scenarios represent increasingly complex dynamical settings, ranging from simulation to hardware implementation. Across all scenarios, the QPGP-based approach consistently outperforms standard ILC and competing GP-based predictive ILC methods, both under nominal conditions and in the presence of external disturbances.
\begin{figure}
    \centering
\includegraphics[width=0.85\linewidth]{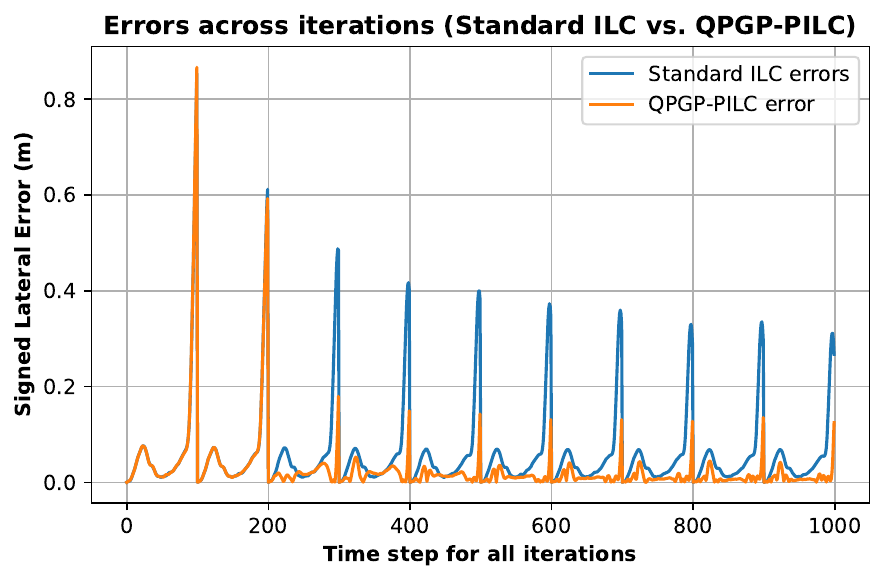}
    \caption{Our key insight is that errors in iterative learning control evolve in a quasi-periodic manner. Incorporating predictive modeling exploiting this structure this into ILC (QPGP-PILC) results in noticeably reduced error magnitude and improved convergence behaviour over standard ILC.\vspace{-5mm}}

    \label{fig:std_ilc_errors}
\end{figure}
The core contributions of our work are:
\begin{itemize}
    \item Showing that iterative learning control results in a quasi-periodic dynamical system (see Fig. \ref{fig:std_ilc_errors}).
    \item A computationally tractable QPGP-based predictive iterative control method that enables continual learning and ongoing, rapid adaptation to changing environments.
    \item A contraction result outlining controller design requirements to ensure mean error convergence when using QPGP predictive iterative learning control.
\end{itemize}

\section{Background}
\subsection{Iterative Learning Control (ILC)} \label{s2}

Consider a system performing a repetitive task over a discrete time interval $t \in \{1,2,\dots,p \}$. Let $\mathbf{u}_i$ denote a control vector of dimension $mp$, and $\mathbf{y}_i$ a system output vector of dimension $np$, respectively. Here, $m$ denotes the control input dimension at each timestep, and $n$ the system output dimension at each timestep. These vectors are formed by stacking all $p$ control inputs for each  of the $k$ control input dimensions, $k = 1,\ldots,m$,  and all $p$ system outputs for each of the $j$ system output dimensions, $j = 1,\ldots,n$, into a lifted form. $i$ corresponds to the $i^\text{th}$ iteration of this repetitive task. The discrete-time closed-loop dynamics of the system over iterations (or trials) can be modeled as follows.
\begin{equation}
    \mathbf{y}_i= g(\mathbf{u}_i)+\mathbf{z}_i, \text{ for all }i \ge 1,\label{system_response}
\end{equation}
where $g:\mathbb{R}^{mp} \to\mathbb{R}^{np}$ is assumed to be continuously differentiable, and is often unknown. Here, $\mathbf{z}_i$ models the disturbances in these dynamics, a zero mean Gaussian process.

Let the tracking error at iteration $i$ be defined as:
\begin{equation}
    \mathbf{e}_{i}=\mathbf{y}_d-\mathbf{y}_{i}, \text{ for all } i\ge1.\label{tracking_error}
\end{equation}
ILC seeks to update the control input for the next iteration to reduce this error. A standard ILC update is given by:
\begin{equation}
   \mathbf{u}_{i+1}=\mathbf{u}_i +\mathbf{L}_i \mathbf{e}_i,\ \forall i\ge 1. \label{ilc_standard_general}
\end{equation}
where $\mathbf{L}_i$ is the learning gain matrix  of dimension $mp \times np$ that determines the influence of the previous iteration’s error $\mathbf{e}_i$ on the updated input.

From \eqref{system_response} and \eqref{ilc_standard_general}, using Taylor's linearization on $g$, we have, 
\begin{align}
    \mathbf{e}_{i+1}&=\mathbf{y}_{d}-\mathbf{y}_{i+1}  =\mathbf{y}_{d}-\left[g(\mathbf{u}_{i})+\mathbf{G}_i\mathbf{L}_i \mathbf{e}_{i}+\mathbf{z}_{i+1}\right]\nonumber\\
    &=(\mathbf{I}_{np}-\mathbf{G}_i\mathbf{L}_i)\mathbf{e}_{i}+ (\mathbf{z}_{i}-\mathbf{z}_{i+1}) \label{e_i_recursion}
\end{align}
where \begin{equation}
    \mathbf{G}_i=\frac{\partial g}{\partial \mathbf{u}}\Bigg \rvert_{\mathbf{u}=\mathbf{u}_i}. \label{jacobian}
\end{equation}

A key insight motivating this work is that ILC generates quasi-periodic error signals, as shown in \eqref{e_i_recursion}. Across each iteration, the error is a perturbed version of the previous error. We demonstrate this graphically in an autonomous racecar path-following task, where steering is the control input and the goal is to minimize tracking error to a fixed reference across iterations. As shown in Fig.~\ref{fig:std_ilc_errors}, this produces quasi-periodic tracking errors across iterations that can be exploited for prediction.

\subsection{Predictive ILC} \label{s3}
Despite its effectiveness, standard ILC updates control inputs reactively using only past tracking errors, which can slow convergence when error patterns repeat or evolve predictably~\cite{survey_ilc}. Predictive ILC (PILC) extends standard ILC by using a model to forecast future errors~\cite{pred_ilc,Huang01012003}. Instead of relying only on past errors, the controller uses predicted errors for the upcoming iteration, $\hat{\mathbf{e}}_{i+1}$, to proactively refine the input~\cite{owens1999prediction}. This explicit compensation accelerates convergence and improves tracking, especially when error patterns repeat or have temporal structure~\cite{owens1999prediction}. The predictive ILC update, similar to \eqref{ilc_standard_general}, is given by 
\begin{equation}
    \mathbf{u}_{i+1} = \mathbf{u}_i +\mathbf{L}_i\mathbf{e}_i + \mathbf{K}_i\hat{\mathbf{e}}_{i+1},
    \label{ilc_predictive_general}
\end{equation}
with $\mathbf{K}_i$ is a predictive learning matrix of dimension $mp \times np$ and $\hat{\mathbf{e}}_{i+1}$ the predicted output error for the $(i+1)^{\text{st}}$ iteration.

Predictive ILC often uses linear models \cite{linear_pilc} from system identification for simplicity and tractability, while nonlinear approximators like neural networks~\cite{PATAN2020445} handle complex dynamics. Gaussian Processes (GPs) with RBF kernels~\cite{gp_ilc} leverage their probabilistic nature to capture structured error patterns and quantify uncertainty, enabling reliable error forecasts, rapid convergence, and improved robustness across iterations. However, GP models face challenges in online or continual learning because the computational cost of parameter estimation and prediction grows cubically with dataset size for standard GP inference (although this can be reduced for Toeplitz matrices~\cite{toeplitz}).
To maintain tractability, strategies like data sub-sampling to inducing points~\cite{sparse_gp1, sparse_gp2, sparse_gp3}, using basis vectors~\cite{HUBER201485}, or discarding past observations are employed. However, these approaches can reduce information and compromise predictive accuracy.

\section{Methodology} \label{s4}
Having shown the inherent quasi-periodicity in ILC, we propose to employ a recently proposed Quasi-Periodic Gaussian Process (QPGP \cite{qpgp_icassp,qpgp_arxiv}) 
to forecast the next-iteration error trajectory in a predictive ILC framework. To the best of our knowledge, this is the first predictive ILC framework that exploits quasi-periodic structure across iterations. The framework admits mean and covariance contraction results, and scales independently of the number of past iterations.

\subsection{Quasi-Periodic Gaussian Processes (QPGP)}\label{s4.1}

A Gaussian process is referred to as QPGP if its sample paths exhibit a quasi-periodic pattern. The covariance structure of a QPGP captures both within-period and between-period variation.  \cite{quasiperiodic} models within-period correlations with a periodic kernel and between-period correlations with geometric decay. However, the prediction of this QPGP is computationally expensive (see \cite{qpgp_arxiv}). One approach to reduce complexity for GP kernels is to view these from a state-space or dynamical systems perspective \cite{state_space_kernel}. 

Along these lines, \cite{qpgp_arxiv} proposed a family of dynamical-system-based QPGP that permits general periodic kernels to model the within-period correlation of the QPGP.
\cite{qpgp_arxiv} also provides a computationally efficient parameter estimation and prediction of the QPGP. In this framework, the within period blocks of a vector $\boldsymbol{\mathsf{x}}_i$, for $i \ge1$ satisfy
\begin{equation}
    \boldsymbol{\mathbf{x}}_{i+1}=\omega \boldsymbol{\mathbf{x}}_{i}+\boldsymbol{\mathbf{\epsilon}}_{i+1} \label{qpgp_str_eqn}
\end{equation}
where $\omega\in(-1,1)$ captures inter-iteration correlation, and $\boldsymbol{\mathbf{\epsilon}}_{i}$ are i.i.d. zero-mean Gaussian vectors.
Stacking $[\boldsymbol{\mathbf{x}}_{1}^\top\ \boldsymbol{\mathsf{x}}_{2}^\top\dots]^\top$ exhibits quasi-periodic behavior. 
This formulation allows prediction of the next iteration using only the most recent error, enabling efficient online updates~\cite{qpgp_arxiv}.

We propose to model the evolving error dynamics under the standard ILC law in \eqref{e_i_recursion} independently for each of the $j=1,2,\dots,n$ output dimensions using \eqref{qpgp_str_eqn}, assuming the $\mathbf{I}_{np}-\mathbf{G}_i\mathbf{L}_i$ term in \eqref{e_i_recursion} is approximately diagonal:
\begin{equation}
    \mathbf{e}_{i+1} \approx (\Omega \otimes I_p)\mathbf{e}_i + \varepsilon_{i+1} \label{eq:approx}
\end{equation}
with  $\varepsilon_{i+1}\sim \mathcal{N}(\mathbf{0},\mathrm{blkdiag}\{\boldsymbol{\Kappa}_1,\ldots,\boldsymbol{\Kappa}_n\})$.
This is reasonable in robotic systems with weak cross-coupling or after suitable coordinate transformation or feedback diagonalization by the inner loop controller, and it simplifies learning gain design while maintaining computational efficiency (Sections~\ref{s4.3}, \ref{s4.2}). Although the disturbance in \eqref{e_i_recursion} is temporally correlated, the dependence typically lasts only one iteration. We model the noise using a zero-mean i.i.d. Gaussian vector with covariance $\boldsymbol{\Kappa}_j$ to provide a simpler representation that captures the dominant and unmodeled residual noise characteristics without introducing unnecessary complexity.

\subsection{Incorporating QPGP in ILC: Prediction} \label{s4.3}
For $j=1,2,\ldots,n$, the error in $j^\text{th}$ output dimension at iteration $i$ corresponds to the $p$-dimensional subvector of $\mathbf{e}_i$ from the index $(j-1)p+1$ to $jp$ and is denoted as follows. 
\begin{align*}
    \mathbf{e}_{i,j}=\mathbf{e}_{i}[(j-1)p+1:jp] =[e_{i,j}(1), e_{i,j}(2), \dots, e_{i,j}(p)]^\top.
\end{align*}
A prediction of each of these, i.e. $\hat{\mathbf{e}}_{i,j}$ is required in \eqref{ilc_predictive_general}. Depending on the compute requirements of the control problem, we can exploit two prediction strategies, given by \cite{qpgp_arxiv}. \textit{Element-wise prediction} predicts each element $\hat{e}_{i+1,j}(t)$ sequentially for all $t = 1, \dots, p$, while \textit{Block prediction} predicts the entire error  $\mathbf{e}_{i+1,j}$ simultaneously. 

\noindent \textbf{Element-Wise Prediction: }In the element-wise prediction approach, given by  \cite{qpgp_arxiv}, each element $e_{i+1}(t)$ can be predicted sequentially
\begin{equation}
\resizebox{0.89\linewidth}{!}{$
\hat{e}_{i+1,j}(t)= \omega_j e_{i,j}(t) + \boldsymbol{\Kappa}_{j;1,t-1}\boldsymbol{\Kappa}_{j;t-1}^{-1} \left(\mathbf{e}_{i+1,j}^{(t-1)} - {\omega}_j\mathbf{e}_{i,j}^{(t-1)}\right)
$}
\label{qpgp_element_predict}
\end{equation}
The vector ${\boldsymbol{\Kappa}}_{j;1,t-1}$ contains the first $t-1$ elements of the first row of ${\boldsymbol{\Kappa}}_j$, while  ${\boldsymbol{\Kappa}}_{j;t-1}$ represents the submatrix formed by the first $t-1$ rows and columns of ${\boldsymbol{\Kappa}}_j$. The terms $\mathbf{e}_{i+1,j}^{(t-1)}$ and $\mathbf{e}_{i,j-1}^{(t-1)}$ correspond to the first $t-1$ elements of $\mathbf{e}_{i+1,j}$ and $\mathbf{e}_{i,j}$ respectively. This expression provides the best predictor of $e_{i+1,j}(t)$ as the conditional mean of $e_{i+1,j}(t)$ given the previous observations  in the same iteration $e_{i+1,j}(t-1),\dots, e_{i+1,j}(1)$ and the observations from previous iterations, $\mathbf{e}_{i,j}$, $\mathbf{e}_{i-1,j}$, \dots, $\mathbf{e}_{1,j}$. For $t=1$, the second term vanishes and $\hat{e}_{i+1,j}(1)=\omega_j e_{i,j}(1)$.

\noindent \textbf{Block Prediction: }In the block prediction approach, the prediction is obtained using
\begin{equation}
    \hat{\mathbf{e}}_{i+1,j}= \omega_j{\mathbf{e}}_{i,j}. \label{qpgp_block_predict}
\end{equation}
This prediction represents the best predictor of $\mathbf{e}_{i+1,j}$, derived as the conditional mean of $\mathbf{e}_{i+1,j}$ given the observations in the previous iterations, $\mathbf{e}_{i,j}$, $\mathbf{e}_{i-1,j}$, \dots, $\mathbf{e}_{1,j}$. 
This is computationally efficient, but does not exploit the covariance structure within an iteration, providing a simple linear prediction.

Notably, these formulations depend only on the immediately preceding error block $\mathbf{e}_{i,j}$, making them computationally lightweight and tractable for online learning settings like ILC. Specifically, the computational complexity of \eqref{qpgp_block_predict} is $\mathcal{O}(p)$, while that of \eqref{qpgp_element_predict} is $\mathcal{O}(p^3)$, where $p$ denotes the number of points in the iteration. In contrast, a conventional GP approach requires $\mathcal{O}(i^2 p^3)$ computations at iteration $i$, which becomes increasingly costly for iterative, online applications. In general, element-wise prediction should be preferred for accuracy, unless the length of each trial prohibits computation at the required control frequency.

It is worth noting that an alternative computationally efficient approach to GP modeling is the use of sparse GPs \cite{sparse_gp2}, which approximate a full GP using $M$ inducing points. For prediction, this reduces the computational complexity to $\mathcal{O}(p M^2)$. When $M>p$, both element-wise and block QPGP-based prediction offer reduced computation. However, even with $M < p$, sparse GPs may be computationally more expensive than a QPGP, as selecting an optimal set of inducing points can itself be challenging; methods such as greedy selection \cite{sparse_greedy}, variational optimization \cite{sparse_varopti}, or k-means based initialization are commonly used, and cross-validation over multiple choices of $M$ may be required, which further increases computational cost. While sparse GPs involve information loss that QPGPs avoid, they could nonetheless be used to model the within-block covariance in the block QPGP framework, offering additional flexibility and efficiency. As shown in our experimental results, element-wise QPGPs generally provide higher accuracy, and block-wise QPGPs achieve better convergence with significantly lower computation time.

\subsection{Incorporating QPGP in ILC: Stability} \label{s4.4}
In predictive ILC, the goal is to anticipate the next iteration’s tracking error rather than relying only on past errors. To achieve this, we use the QPGP  to model (Section~\ref{s4.1}) and predict (Section~\ref{s4.3}) the evolution of error trajectories across iterations. To characterize the behaviour of the proposed QPGP-based predictive ILC, the following theorem presents stability results for oracle prediction models.

\begin{theorem} \label{convergence_theorem}
    Consider the nonlinear discrete-time system given in \eqref{system_response} and define the tracking error as in \eqref{tracking_error}. Assume $g$ is locally linearizable around the current input with its Jacobian given in \eqref{jacobian}. Then, for the two control strategies,
    \begin{enumerate}
    \item \textit{Element-wise prediction:} The tracking error $\mathbf{e}_i \in \mathbb{R}^{np}$ evolves as 
    \begin{equation}
        \mathbf{e}_{i+1}^{(t)}=\mathbf{B}_i \mathbf{e}_i^{(t)}+(\mathbf{z}_i^{(t)}-\mathbf{z}_{i+1}^{(t)}),\forall \  t=1,\dots,p \label{element_recursion}.
    \end{equation}
where, $\mathbf{B}_i=\mathbf{I}_{nt}-\mathbf{G}_i^{(t)}\mathbf{L}_i^{(t)}-\mathbf{G}_i^{(t)} \mathbf{K}_i^{(t)} \mathcal{M}_i^{(t)}$, where $\mathcal{M}_i^{(t)}=\bigoplus_{j=1}^n \mathbf{M}_{i,j}^{(t)}$, $\mathbf{M}_{i,j}^{(1)}=\omega_j$ and for $t=2,\dots,p$, \begin{equation*}
    \mathbf{M}_{i,j}^{(t)} =
    \begin{bmatrix}
    \mathbf{M}_{i,j}^{(t-1)} & \mathbf{0} \\
    \boldsymbol{\Kappa}_{j;1,t-1} \, \boldsymbol{\Kappa}_{j;t-1}^{-1} \big( \mathbf{M}_{i,j}^{(t-1)} - \omega_j \mathbf{I}_{t-1} \big) & \omega_j
    \end{bmatrix}.
\end{equation*}
and $\mathbf{G}_i^{(t)}$, $\mathbf{L}_i^{(t)}$ and $\mathbf{K}_i^{(t)}$ are sub-matrices formed by the first $t$ columns of matrices $\mathbf{G}_i$, $\mathbf{L}_i$ and $\mathbf{K}_i$ respectively. Further, if there exists a constant $\gamma_E \in [0,1)$ and $C>0$, such that $\sup_i || \mathbf{B}_i||_2 \le \gamma_E$ and $\sup_j || \boldsymbol{\Kappa}_j ||_2<C/2$ then for all $t=1,\dots,p$,
\begin{align}
    \lim_{i \to \infty}E\left(\mathbf{e}_{i}^{(t)}\right)=\mathbf{0}, \text{ and } \lim_{i \to \infty}\left \lvert \left \lvert \text{Cov}\left(\mathbf{e}_{i}^{(t)}\right)\right \rvert \right \rvert_2 < C.
\end{align}
        \item \textit{Block prediction:} The tracking error $\mathbf{e}_i \in \mathbb{R}^{np}$, for the $i^{\text{th}}$ iteration, evolves as  \begin{equation}
        \mathbf{e}_{i+1}=\mathbf{A}_i \mathbf{e}_i+(\mathbf{z}_i-\mathbf{z}_{i+1}) \label{block_recursion}.
    \end{equation}
where, $\mathbf{A}_i=\mathbf{I}_{np}-\mathbf{G}_i\mathbf{L}_i-\mathbf{G}_i \mathbf{K}_i (\Omega \otimes \mathbf{I}_p)$ and where $\Omega=\text{diag}\{\omega_1,\dots,\omega_n\}$. Further, if there exists a constant $\gamma_B \in [0,1)$ and $C>0$, such that $\sup_i || \mathbf{A}_i||_2 \le \gamma_B$ and $\max_{1 \le j \le n}
|| \boldsymbol{\Kappa}_j ||_2<C/2$ then
\begin{align}
    \lim_{i \to \infty}E(\mathbf{e}_{i})=\mathbf{0}, \text{ and } \lim_{i \to \infty}||\text{Cov}(\mathbf{e}_{i})||_2 < C
\end{align}
        
\end{enumerate}
\end{theorem}

\begin{proof}[\unskip\nopunct]
\renewcommand{\qedsymbol}{}
Using \eqref{system_response} and \eqref{ilc_predictive_general}, and using the Taylor's linearization on $g$, we have, 
\begin{align}
    \mathbf{e}_{i+1}&= (\mathbf{I}_{np}-\mathbf{G}_i\mathbf{L}_i)\mathbf{e}_{i}- \mathbf{G}_i \mathbf{K}_i\hat{\mathbf{e}}_{i+1}+(\mathbf{z}_{i}-\mathbf{z}_{i+1}) \label{e_i_recursion_pred}
\end{align}

\noindent For 1) \textit{Element-wise prediction:} We have, from \eqref{qpgp_element_predict}, 
    \begin{equation*}
       \hat{\mathbf{e}}_{i+1}^{(t)}= \mathcal{M}_i^{(t)}\mathbf{e}_{i}^{(t)}
    \end{equation*}
Then, for $t=1,2,\dots,p$,
\begin{align*}
    \mathbf{e}_{i+1}^{(t)}&= \left(\mathbf{I}_{nt}-\mathbf{G}_i^{(t)}\mathbf{L}_i^{(t)}-\mathbf{G}_i^{(t)} \mathbf{K}_i^{(t)} \mathcal{M}_i^{(t)}\right)\mathbf{e}_{i}^{(t)}\\
    &+\left(\mathbf{z}_{i}^{(t)}-\mathbf{z}_{i+1}^{(t)}\right) 
\end{align*}
Then, as $E\left(\mathbf{z}_i^{(t)}\right)=\mathbf{0}$, for all $i \ge1$, then
\begin{equation}
    E\left(\mathbf{e}^{(t)}\right)=\mathbf{B}_i^{(t)} E\left(\mathbf{e}_{i}^{(t)}\right)
\end{equation}
If there exists $\gamma_E \in [0,1)$ such that $\sup_i\big \lvert \big \lvert \mathbf{B}_i^{(t)}\big \rvert \big \rvert_2 \le \gamma_E$ then 
\begin{equation}
    E\left(\mathbf{e}_{i+1}^{(t)}\right)\le \gamma_B^i E\left(\mathbf{e}_{i}^{(t)}\right) \to \mathbf{0} 
\end{equation}
Let $\mathbf{P}_i^{(t)}=\text{Cov}\left(\mathbf{e}_i^{(t)}\right)$. Then, 
\vskip-5mm
\begin{align*}
\mathbf{P}_{i+1}^{(t)}&=\mathbf{B}_i^{(t)}\mathbf{P}_i^{(t)} {\mathbf{B}_i^{(t)}}^\top +  2\bigoplus_{j=1}^{n} \boldsymbol{\Kappa}_{j;t}
\end{align*}
\vskip-3mm
\noindent Taking the norm on both sides, we have
\begin{equation}
    \big \lvert \big \lvert \mathbf{P}_{i+1}^{(t)} \big \rvert \big \rvert_2\le\gamma_E^{2i}  \big \rvert \big \rvert\mathbf{P}_i \big \rvert \big \rvert_2+ \max_{1\le j\le n} \big \lvert \big \lvert \boldsymbol{\Kappa}_{j;t} \big \rvert \big \rvert_2
\end{equation}
As $\gamma_E\in [0,1)$ and $||\boldsymbol{\Kappa}_{j;t}||_2<C/2, $ for all $j=1,\dots, n$, 
\begin{equation}
         \lim_{i \to \infty} \left \lvert \left \lvert\text{Cov}\left(\mathbf{e}_{i}^{(t)}\right) \right \rvert \right \rvert_2<C, \text{ for all }t=1,2,\dots,p.
\end{equation}

\noindent For 1) \textit{Block prediction:} we have, 

        $\hat{\mathbf{e}}_{i+1}= (\Omega \otimes \mathbf{I}_p)\mathbf{e}_{i}$.

Then 
\begin{equation}
    \mathbf{e}_{i+1}= (\mathbf{I}_{np}-\mathbf{G}_i\mathbf{L}_i-\mathbf{G}_i \mathbf{K}_i (\Omega \otimes \mathbf{I}_p))\mathbf{e}_{i}+(\mathbf{z}_{i}-\mathbf{z}_{i+1}) \label{e_i_recursion_block}
\end{equation}
Then, as $E(\mathbf{z}_i)=\mathbf{0}$, for all $i \ge1$, then
\begin{equation}
    E(\mathbf{e}_{i+1})=\mathbf{A}_i E(\mathbf{e}_{i})
\end{equation}
If there exists $\gamma_B \in [0,1)$ such that $\sup_i|| \mathbf{A}_i||_2 \le \gamma_B$ then 
\begin{equation}
    E(\mathbf{e}_{i+1})\le \gamma_B^i E(\mathbf{e}_{i}) \to \mathbf{0} 
\end{equation}
Let $\mathbf{P}_i=\text{Cov}(\mathbf{e}_i)$. Then, \vskip-5mm
\begin{align*}
\mathbf{P}_{i+1}&=\mathbf{A}_i\mathbf{P}_i \mathbf{A}_i^\top +  2\bigoplus_{j=1}^{n} \boldsymbol{\Kappa}_j
\end{align*}
Taking the norm on both sides, we have
\begin{equation}
    || \mathbf{P}_{i+1}||_2\le\gamma_B^{2i} ||\mathbf{P}_i||_2+ \max_{1\le j\le n}|| \boldsymbol{\Kappa}_j||_2
\end{equation}
As $\gamma_B\in [0,1)$ and $||\boldsymbol{\Kappa}_j||_2<C/2, $ for all $j=1,\dots, n$, 
\begin{equation}
         \lim_{i \to \infty}||\text{Cov}(\mathbf{e}_{i})||_2<C
\end{equation}
\end{proof} \vskip-5mm
The theorem first establishes stability at the level of the mean behavior of the tracking error. When the lifted iteration dynamics are contractive, the expected tracking error decreases from one iteration to the next, demonstrating that the controller consistently corrects past mistakes and refines the control input. As a result, the system learns the desired trajectory in an average sense despite the presence of modeling inaccuracies and stochastic perturbations. This guarantees that the learning mechanism is not merely reactive but progressively improves performance over repeated trials.

At the second-moment level, the theorem ensures that the uncertainty associated with the tracking error does not grow unbounded. By requiring a uniform spectral bound on the kernel covariance matrices, the predictive component is prevented from amplifying stochastic disturbances as the iterations evolve. Consequently, the error covariance remains bounded, which implies mean-square stability of the learning process and ensures that variability around the desired trajectory stays controlled rather than escalating over time.

As is standard in ILC, practical implementation the contraction condition and covariance bounds translate directly into criteria for selecting the learning gain $\mathbf{L}_i$ and the predictive gain $\mathbf{K}_i$ for $i^{\text{th}}$ iteration. When $\mathbf{G}_i$ is known, gains can be tuned relatively aggressively so that the induced iteration operator remains contractive while the associated covariance operators are uniformly bounded. If this is not known, more cautious gain choices can be made (eg. by decaying the learning rate over time).

\subsection{Incorporating QPGP in ILC: Parameter Estimation } \label{s4.2}

We adopt the computationally efficient method of obtaining consistent estimators of the parameters $\omega_j$ and $\boldsymbol{\Kappa}_j$ from \cite{qpgp_arxiv}. This uses a two-stage algorithm, for every output dimension $j=1,\dots,n$. In Stage I, an iterative, alternating minimization is applied to a reduced likelihood function, which excludes the  marginal contribution of the initial error block $\mathbf{e}_{1,j}$. This iteratively estimates $\omega_j$ and an unconstrained covariance matrix, $\tilde{\boldsymbol{\Kappa}}_{j}$. In Stage II, the estimated covariance matrix $\tilde{\boldsymbol{\Kappa}}_j$ is projected onto the set of valid periodic covariance kernels by averaging along the diagonals and ensuring positive definiteness through spectral truncation  \cite{peterhall}. The resulting estimates, $\hat{\omega}_j$ and $\hat{\boldsymbol{\Kappa}}_j$ are consistent, converging in probability to the true parameters as the number of iterations increase. Further, if the covariance function is from a known parametric family, the hyperparameters of the covariance function can be calculated using 
\begin{equation}\label{parameteric_kappa}
(\hat{\sigma}_j^2, \hat{\boldsymbol{\theta}}_j)=\text{arg min}_{(\sigma^2, \boldsymbol{\theta})\in (0,\infty)\times \Theta}\|\tilde{\boldsymbol{\Kappa}}_j-\boldsymbol{\Kappa}(\sigma^2, \boldsymbol{\theta})\|_{F}.
\end{equation}
In the absence of prior structural knowledge of the parametric family, the estimate $\hat{\boldsymbol{\Kappa}}_j$ is retained directly and treated as a \textit{general kernel} obtained purely through a data-driven process.

In the predictive ILC setting, the two-stage estimation procedure naturally integrates into the iterative update framework. After iteration $i$, the new tracking errors 
$\mathbf{e}_{i,j}$
are incorporated into the training dataset, for all $j=1,2,\dots,n$. The aggregated training data up to iteration $i$ is thus given by $ [\, \mathbf{e}_{1,j}^\top, \dots, \mathbf{e}_{i,j}^\top \,]^\top$ for output index $j$.
Using this updated dataset, the parameters are re-estimated via the two-stage procedure to obtain $\hat{\omega}^{(i)}_j$ and $\hat{\boldsymbol{\Kappa}}^{(i)}_j$.  We initialize the estimation process with parameters from the previous trial. 
These updated estimates are then employed for prediction after each iteration. 

Since prediction is implemented using plug-in estimates $\hat{\omega}^{(i)}_j$ and $\hat {\boldsymbol{\Kappa}}^{(i)}_j$, finite-sample estimation error can introduce predictor mismatch relative to the ideal QPGP conditional mean. Accordingly, when tuning $\mathbf{L}_i$ and $\mathbf{K}_i$ in practice, we recommend maintaining a nontrivial contraction margin in Theorem~2 (e.g., selecting gains so that $\|\mathbf{A}_i\|_2$ or $\|\mathbf{B}_i\|_2$ are comfortably below $1$) to ensure robust stability under this mismatch.

\section{Experiments}
We compare the performance of the QPGP-based predictive ILC (QPGP-PILC) framework with Standard ILC and GP-based predictive ILC (GP-PILC) using three scenarios.

\textbf{A vehicle} following a predefined path, with QPGP-PILC anticipating recurring steering errors and correcting these via block- and element-wise predictions using \eqref{qpgp_block_predict} and \eqref{qpgp_element_predict}.

\textbf{A three-link robotic manipulator} tracking a repetitive trajectory, where QPGP-PILC improves tracking accuracy and convergence. An external disturbance midway evaluates the method’s robustness against unexpected perturbations.

\textbf{A Stretch Robot} tracking problem,  which shows QPGP-PILC effectively handles uncertainties and actuator noise while achieving faster convergence and lower tracking errors.

For each case, we analyze tracking error, convergence speed, and computational efficiency, demonstrating the clear advantages of QPGPs for predictive error correction in ILC.

\subsection{Vehicle Trajectory Tracking} \label{s5}
Using the experimental procedure in Appendix~\ref{appendix:vehicle}, we evaluate convergence and tracking performance of Standard ILC, GP-PILC, and QPGP-PILC controllers. The feedforward steering input $(\delta_k^{\mathrm{ff}})$ is updated iteratively according to each ILC variant \eqref{ilc_standard_general} or \eqref{ilc_predictive_general}, allowing assessment of predictive modeling effects on error reduction and trajectory accuracy.

Standard ILC updates the input solely using the previous iteration’s tracking error \eqref{ilc_standard_general}. Two predictive ILC variants are considered \eqref{ilc_predictive_general}. GP-PILC employs a Gaussian Process with an RBF kernel to model the expected error. Although it converges faster than Standard ILC, it requires the complete error history, resulting in high computational cost (Table~\ref{car_time}). A Sparse GP with $M=100$ optimized inducing points reduces computation, achieving runtimes closer to QPGP-PILC, but its prediction accuracy remains lower.

\begin{table}
    \centering
        \caption{Total Computation time (in secs.) for 100 iterations for Standard ILC, GP-PILC, Sparse GP-PILC and QPGP-PILC (block and element-wise).} {\begin{tabular}{|c|c|c|c|c|c|}
    \hline
        \textbf{Approach}& Standard &  QPGP & QPGP &Sparse  &GP  \\
        & & (block)& (element) &GP&\\
        \hline
         \textbf{Time (secs.)}&0.23&18.74& 19.31&19.67&649.23 \\ \hline
    \end{tabular}}

    \label{car_time}
\end{table}

 \begin{figure}
    \centering
        
        \includegraphics[width=\linewidth]{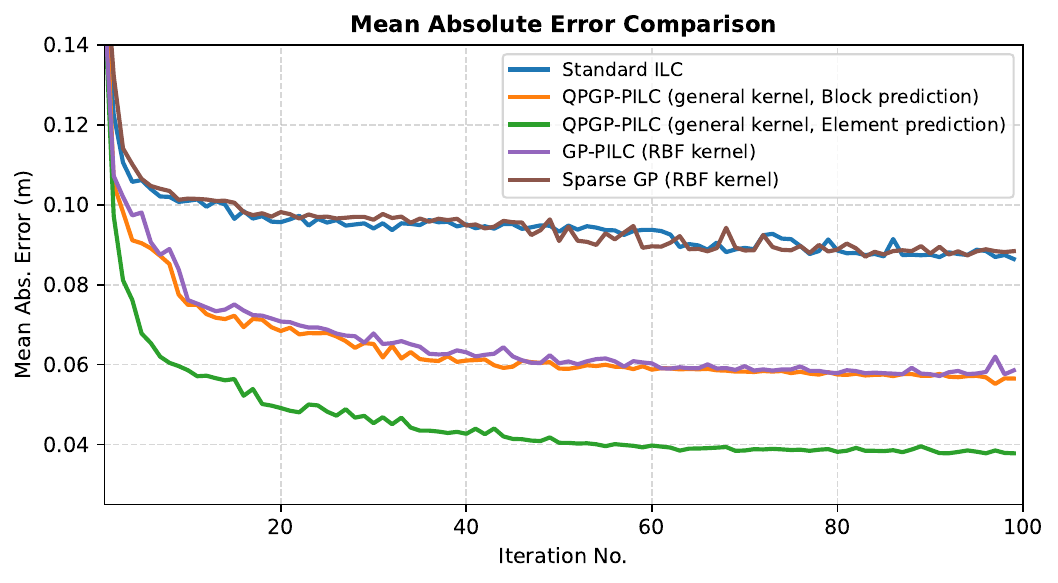}
    \caption{Comparison of error convergence for vehicle trajectory tracking for Standard ILC, GP-PILC, Sparse GP-PILC, and QPGP-PILC (block and element-wise). Element-wise QPGP-PILC converges fastest, followed by block-based QPGP-PILC. GP-PILC shows similar convergence, while Sparse GP-PILC converges slightly slower. QPGP-PILC consistently achieves the most accurate tracking.}
    \label{fig:car_error}
\end{figure}

QPGP-PILC leverages the quasi-periodic structure of the error signal, capturing recurring patterns while using only the most recent error block. Both block- and element-wise predictors are evaluated, consistently achieving higher accuracy than full or sparse GP ILC. Simulation results (Fig.~\ref{fig:car_error}) show GP- and QPGP-based methods converge faster than Standard ILC, with the element-wise QPGP predictor converging faster than the block-based variant at moderate additional cost (Table~\ref{car_time}). Unlike GP-based methods, QPGP-PILC avoids storing full error histories, greatly reducing computation.

At the 50th iteration (Fig.~\ref{fig:car_traj}), QPGP-PILC trajectories align closely with the reference path, whereas Standard ILC shows visible deviations, demonstrating QPGP-PILC’s ability to exploit quasi-periodic patterns for faster convergence and improved tracking accuracy.

 \begin{figure}
    \centering
    \includegraphics[width=0.8\linewidth]{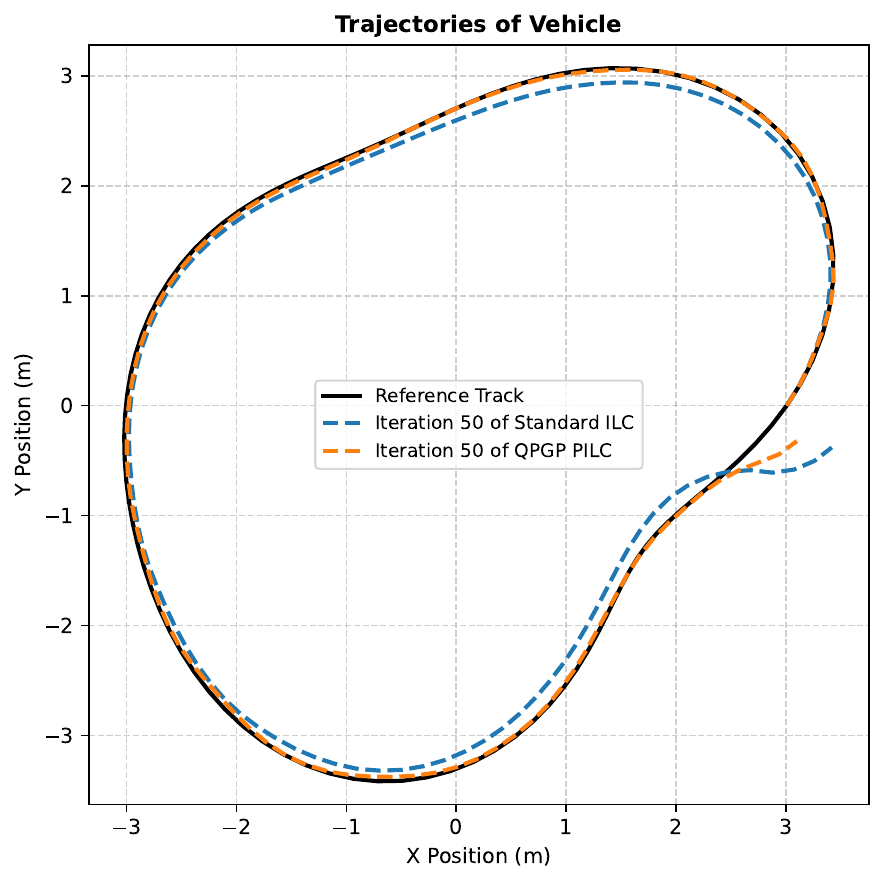}
    \caption{Comparison of vehicle trajectories at the 50th iteration under Standard ILC and QPGP-PILC controllers. The QPGP-PILC trajectory closely follows the reference, whereas the Standard ILC trajectory shows visible deviations.}
    \label{fig:car_traj}
\end{figure}

\subsection{Manipulator Trajectory Tracking} \label{s6}
\begin{figure}
    \centering
    \includegraphics[width=\linewidth]{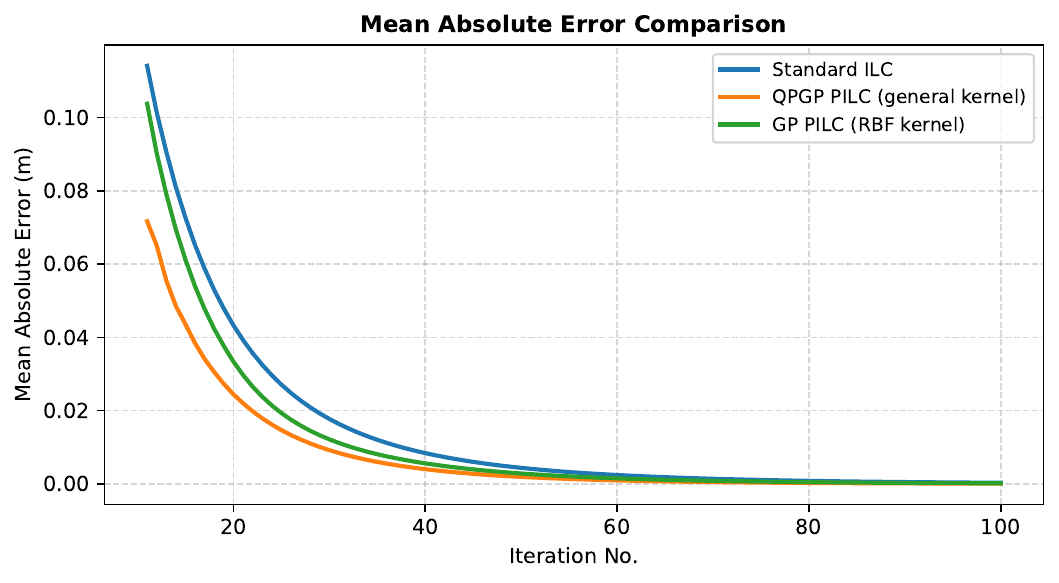}
    \caption{Comparison of error convergence for manipulator between Standard ILC, QPGP-PILC, and GP-PILC. QPGP-PILC achieves the fastest convergence, followed by GP-PILC and then Standard ILC. In addition to its superior convergence rate, QPGP-PILC is computationally more efficient, as it requires only the most recent iteration’s errors for predictions.}
    \label{fig:manipulator_error}
\end{figure} 

\begin{figure}
    \centering
    \includegraphics[width=\linewidth]{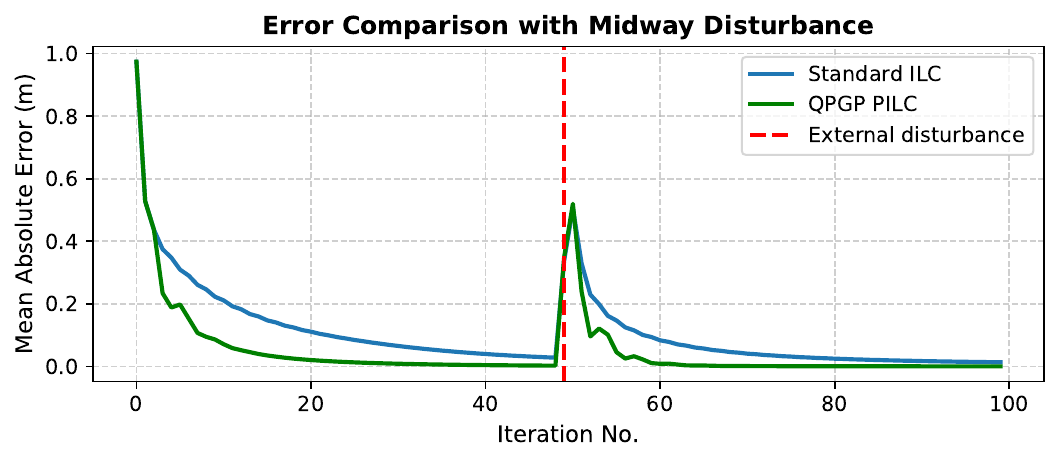}
    \caption{Response of the manipulator to a mid-iteration disturbance under Standard ILC and QPGP-PILC. Standard ILC reacts slowly, while QPGP-PILC anticipates and corrects deviations, recovering to pre-disturbance error levels much faster.}
    \label{fig:bias_manipulator}
\end{figure}

The manipulator setup is detailed in Appendix~\ref{appendix:manipulator}. Feedforward joint commands are updated through the Jacobian using \eqref{ilc_standard_general} for Standard ILC and \eqref{ilc_predictive_general} for predictive ILC.

Standard ILC updates joint angles solely using the previous iteration’s errors, reducing errors slowly. Predictive ILC methods, including GP- and QPGP-PILC, model error dynamics to anticipate deviations. GP-PILC uses the full history of joint errors, whereas QPGP-PILC leverages quasi-periodic patterns and requires only the most recent block for prediction. 

Simulations show that both GP- and QPGP-PILC converge faster than Standard ILC (Fig \ref{fig:manipulator_error}), with QPGP-PILC achieving the fastest error reduction and closely tracking the reference trajectory.
This highlights the benefit of predictive probabilistic modeling in multi-joint learning using QPGPs.

A mid-iteration disturbance was introduced to test robustness. Standard ILC reacts slowly, taking several iterations to recover. QPGP-PILC, using learned quasi-periodic error patterns, quickly corrects the perturbation and restores pre-disturbance accuracy, demonstrating robust and reliable trajectory tracking. (Fig. \ref{fig:bias_manipulator}). This suggests that the diagonal approximation in \eqref{eq:approx} remains effective even in moderately coupled robotic systems.

\begin{figure}
    \centering
    \includegraphics[width=\linewidth]{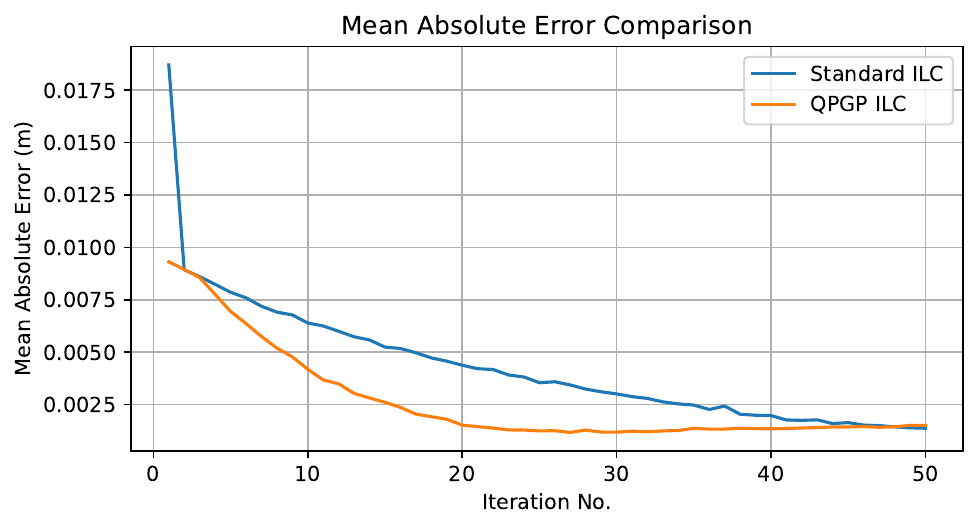}
    \caption{Convergence of tracking error for Hello Robot Stretch 3 over iterations for standard ILC and QPGP-based PILC. The QPGP-PILC achieves faster convergence than Standard ILC.  }
    \label{robot_error}
\end{figure}

\begin{figure}
    \centering
    \includegraphics[width=\linewidth]{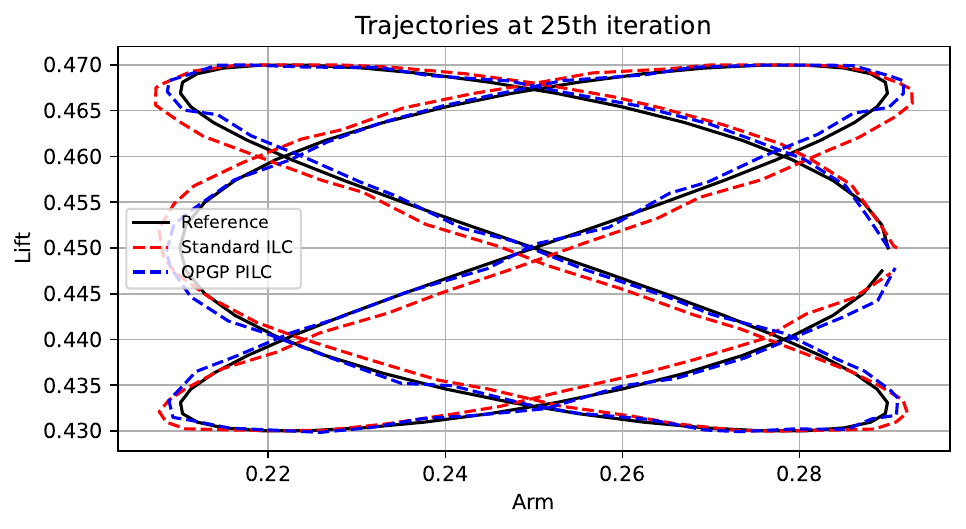}
    \caption{Tracking performance of Hello Robot Stretch 3 at 25th iteration under Standard ILC and QPGP-PILC controllers. The QPGP trajector is closer to the reference path compared to the Standard ILC trajectory, which exhibits noticeable deviations.}
    \label{robot_traj}
\end{figure}

\subsection{Stretch Movement} \label{s7}

ILC has been used to improve robotic trajectory tracking \cite{robot_ilc1,robot_ilc2}. As a final experiment, we demonstrate the use of QPGPs on a real robot, the Hello Robot Stretch 3. The Stretch robot, being a servo- and cable-driven arm with inherent backlash, serves as a realistic platform to test iterative learning control for error compensation. For evaluation, the robot repeatedly executed a Lissajous trajectory using only the lift and arm joints (Appendix~\ref{appendix:robot}). Standard ILC updates inputs reactively from previous errors, while QPGP-PILC anticipates deviations by modeling the quasi-periodic error structure.

QPGP-PILC converges faster than Standard ILC, effectively capturing the repeated Lissajous motion and reducing errors more quickly and stably. Its trajectory remains closer to the reference path, demonstrating improved tracking accuracy and confirming its practical advantages for learning-based control.  

\section{Conclusion} \label{s8}
We proposed a QPGP-based predictive iterative learning control (PILC) framework to improve trajectory tracking in repeated robotic tasks. Analysis of standard ILC errors revealed quasi-periodic patterns not explicitly captured by conventional methods. By integrating quasi-periodic Gaussian processes, the framework estimates system parameters and forecasts future trajectory errors, enabling proactive control adjustments.

QPGP-PILC is well-suited for non-stationary continual learning, handling slowly varying dynamics or environmental changes efficiently using only the previous iteration’s errors. The proposed approach showed faster convergence and improved tracking error across multiple environments. Overall, QPGP-PILC offers a powerful approach for enhancing performance in repetitive robotic and industrial tasks requiring continual  learning with fixed memory and fixed computational complexity.

\bibliographystyle{ieeetr}
\bibliography{refs_ilc}
\appendix
\section{Appendix A: Simulation Environment and Experimental Setup}

The code for all three experiments described in this paper is available at our GitHub repository.\footnote{\url{https://github.com/unnati-nigam/QPGP-PILC}}
\subsection{Vehicle Experimental Setup} \label{appendix:vehicle}

\textbf{Reference Raceline Generation: }A smooth closed racetrack is generated using a polar parameterization with $s \in [0, 2\pi]$ and $p$ uniformly spaced points.
\begin{equation}
R(s) = 10 + 2\sin(2s) + \sin(3s),
\end{equation}
and the corresponding Cartesian coordinates are
\begin{equation}
x_{\mathrm{ref}}(s) = R(s)\cos(s), \quad y_{\mathrm{ref}}(s) = R(s)\sin(s).
\end{equation}
Discretized points are $s_j = \frac{2\pi j}{p}$, $j=1,\dots,p$, and the curve is scaled so the total length equals $L_d = v,\Delta t,p$, where $v$ is vehicle speed and $\Delta t$ is the timestep.

\textbf{Vehicle Dynamics per Lap:} For lap $i$ and step $k$, the vehicle state $(x^{(i)}_k, y^{(i)}_k, \theta^{(i)}k)$ evolves as:
\begin{align}
x^{(i)}_{k+1} &= x^{(i)}_k + v \cos(\theta^{(i)}_k)\,\Delta t, \label{eq:x}\\
y^{(i)}_{k+1} &= y^{(i)}_k + v \sin(\theta^{(i)}_k)\,\Delta t, \label{eq:y}\\
\theta^{(i)}_{k+1} &= \theta^{(i)}_k + \frac{v}{L}\tan(\delta^{(i)}_k)\,\Delta t + d_\theta, \label{eq:theta}
\end{align}
where $v=8$m/s is the constant longitudinal velocity, $W=0.5$m is the wheelbase, and $d_\theta=0.04$ is a small heading drift. The steering input is
\begin{equation}
\delta^{(i)}_k = g\bigl(\delta^{\mathrm{ff},(i)}_k + \delta^{\mathrm{fb},(i)}_k\bigr) + b_0 + b_1 k + \varepsilon_k^{(i)},
\end{equation}
where $g=0.7$ is the steering gain, $b_0=0.15$ is a constant bias, $b_1= 0.01$ is the slope of a slowly increasing drift, and $\varepsilon_k^{(i)} \sim \mathcal{N}(0,\sigma^2)$, with $\sigma^2= 0.015$, is Gaussian noise. The steering angle is saturated to $|\delta^{(i)}_k|\le 0.5$ radians.

\textbf{Feedback Control:} 
The feedback steering $\delta^{\mathrm{fb},(i)}_k$ is computed using a pure pursuit geometric controller \cite{pure_pursuit}. Let $(x_{\mathrm{ref},k+1}, y_{\mathrm{ref},k+1})$ be the next reference point:
\begin{eqnarray}
&\Delta x_k = x_{\mathrm{ref},k+1} - x^{(i)}_k, \ \ \
\Delta y_k = y_{\mathrm{ref},k+1} - y^{(i)}_k,\\
&\psi_k = \tan^{-1}\left(\frac{\Delta y_k}{\Delta x_k}\right), 
\alpha_k = \psi_k - \theta^{(i)}_k,\\
&\delta^{\mathrm{fb},(i)}_k = \tan^{-1}\left(\frac{2 L \sin(\alpha_k)}{\sqrt{\Delta x_k^2 + \Delta y_k^2}}\right).
\end{eqnarray}

\textbf{Lateral Error Computation:}
The signed lateral deviation of the vehicle from the reference trajectory at step $k$ is
\begin{equation}
e^{(i)}_k = \mathbf{n}_\star^\top \left( \begin{bmatrix} x^{(i)}_k \\ y^{(i)}_k \end{bmatrix} - \begin{bmatrix} x_{\mathrm{ref},\star} \\ y_{\mathrm{ref},\star} \end{bmatrix} \right),
\end{equation}
where $\star$ is the index of the nearest reference point and
\begin{equation}
\mathbf{t}_\star = \frac{\begin{bmatrix} x_{\mathrm{ref},\star+1} - x_{\mathrm{ref},\star-1} \\ y_{\mathrm{ref},\star+1} - y_{\mathrm{ref},\star-1} \end{bmatrix}}{\left\| \begin{bmatrix} x_{\mathrm{ref},\star+1} - x_{\mathrm{ref},\star-1} \\ y_{\mathrm{ref},\star+1} - y_{\mathrm{ref},\star-1} \end{bmatrix} \right\|}, \quad
\mathbf{n}_\star = \begin{bmatrix}-t_{y,\star} \\ t_{x,\star}\end{bmatrix}.
\end{equation}

Learning and predictive gains were annealed across iterations according to a $1/i$ schedule. The diminishing step size enabled large corrective updates in the initial laps and progressively smaller adjustments thereafter, improving robustness to noise and promoting stable tracking error convergence.

\subsection{Manipulator Experimental Setup} \label{appendix:manipulator}
\noindent \textbf{Reference Trajectory: }
The manipulator is required to follow a smooth, repetitive trajectory in the Cartesian plane. The desired end-effector path is a circular path defined as:
\begin{equation}
x_{\mathrm{ref}}(s) = 1.5 + 0.5 \cos(s), \ \
y_{\mathrm{ref}}(s) = 1.0 + 0.5 \sin(s)
\end{equation}
where $s \in [0, 2\pi]$ and the trajectory is sampled at $p$ uniformly-spaced discrete points to generate the reference for simulation. 

\noindent \textbf{Manipulator:}
The system is a 3-link planar manipulator operating in the XY-plane. The manipulator has 3 links, $l_1 = 1.0$m, $l_2 = 1.0$m, and $l_3 = 0.5$m. It has three joints with angles $\theta_1, \theta_2, \theta_3$, measured relative to the previous link. 

The forward kinematics of the manipulator, giving the i-th positions of the 3 joints and the end-effector, are
\begin{align}
(x_0,y_0) &= (0,0), \nonumber\\
(x_i,y_i) &= \sum_{k=1}^{i} l_k
\big(\cos(\textstyle\sum_{j=1}^{k}\theta_j),\,
\sin(\textstyle\sum_{j=1}^{k}\theta_j)\big).
\label{eq:forward_kinematics}
\end{align}

The Jacobian matrix, relating joint velocities to end-effector velocities, is given by \eqref{manipulator_jacobian}.
\begin{equation}
J =
\resizebox{0.8\columnwidth}{!}{$
\begin{bmatrix}
 -\displaystyle\sum_{k=1}^{3} l_k\, \sin\!\left(\sum_{s=1}^{k}\theta_s\right) &
 -\displaystyle\sum_{s=2}^{3} l_k\, \sin\!\left(\sum_{s=1}^{k}\theta_s\right) &
 -l_3\, \sin\!\left(\sum_{s=1}^{3}\theta_s\right) \\[6pt]
 \displaystyle\sum_{k=1}^{3} l_k\, \cos\!\left(\sum_{s=1}^{k}\theta_s\right) &
 \displaystyle\sum_{k=2}^{3} l_k\, \cos\!\left(\sum_{s=1}^{i}\theta_s\right) &
 l_3\, \cos\!\left(\sum_{s=1}^{3}\theta_s\right)
\end{bmatrix}
$}
\label{manipulator_jacobian}
\end{equation}

\noindent \textbf{Inverse Kinematics (IK): }
Inverse kinematics computes the joint angles $(\theta_1, \theta_2, \theta_3)$ required to reach a desired end-effector position $(x_{\mathrm{ref}}, y_{\mathrm{ref}})$. For a planar 3-link arm, exact IK can be complex; here the third joint orientation is ignored ($\theta_3 = 0$). The desired wrist position (excluding the last link) is
\begin{equation}
x_e = x_{\mathrm{ref}} - l_3, \quad y_e = y_{\mathrm{ref}}
\end{equation}

The first two joint angles are obtained geometrically:
\begin{align}
D &= \frac{x_e^2 + y_e^2 - l_1^2 - l_2^2}{2 l_1 l_2} \\
\theta_2 &= \tan^{-1}2\Big(\sqrt{1-D^2}, D\Big) \\
\theta_1 &= \tan^{-1}(y_e, x_e) - \tan^{-1}\Big(l_2 \sin \theta_2, l_1 + l_2 \cos \theta_2\Big) 
\end{align}

This provides an approximate initial feedforward joint trajectory for the manipulator to follow the desired path. At each iteration $i$, the manipulator executes the feedforward joint angles with unknown biases. Let \( s \) be a vector of \( p \) equally spaced points in \([0,1]\), then

\resizebox{\columnwidth}{!}{$
\begin{aligned}
b_1(s_t) &= 0.2 + 0.5 \sin(8 \pi s_t) + \epsilon_t^{(1)}, \quad \epsilon_t^{(1)} \sim N(0, 0.1^2),\\[1mm]
b_2(s_t) &= -0.25 + 0.1 \cos(6 \pi s_t) + \epsilon_t^{(2)}, \quad \epsilon_t^{(2)} \sim N(0, 0.2^2),\\[1mm]
b_3(s_t) &= 0.35 + 0.5 \exp\!\Bigg(\!-\frac{(s_t - 0.04)^2}{2(0.05)^2}\!\Bigg) + \epsilon_t^{(3)}, \quad \epsilon_t^{(3)} \sim N(0, 0.1^2).
\end{aligned}
$}
are the biases for each joint angle 
\begin{equation}
\theta_j^\text{actual} = \theta_j^\text{ff} + b_j(s), \quad j = 1,2,3
\end{equation}

\noindent \textbf{Tracking Error:} at the $i^{\text{th}}$ iteration is computed as the Euclidean norm of the position errors 
$
|| \boldsymbol{\mathsf{e}}_i||_2 = 
\left\Vert\begin{matrix} x_{\mathrm{ref}} - x_\text{ee}^{(i)} \\ y_{\mathrm{ref}} - y_\text{ee}^{(i)} \end{matrix}\right\Vert_2.
$
The learning gain was set to $L=0.25$, and the prediction gain to 
$K=0.3$, for all iterations, providing stable and consistent error reduction across iterations.

\subsection{Robot Experimental Setup} \label{appendix:robot}

The Hello Robot Stretch 3 was controlled using \textbf{Python} via the official \texttt{stretch\_body} API\footnote{\url{https://www.hello-robot.com/pages/stretch}}, interacting through the \texttt{Robot} class for high-level commands. The true end-effector positions were obtained from the robot's onboard joint encoders through the API (\texttt{robot.arm.status['pos']} and \texttt{robot.lift.status['pos']}) at each control step. The control loop executed at approximately 10 Hz(\texttt{time.sleep(0.1)} between commands), resulting in an end-to-end latency of roughly 0.1--0.15 s. The robot followed a smooth trajectory within the workspace limits of the arm and lift joints, as defined in the experiment.

\noindent \textbf{Reference Trajectory: }For $t \in [0, 2\pi)$,
\[
y(t) = 0.25 + 0.04 \, \sin(3t + \pi/2),\ \ 
z(t) = 0.45 + 0.02 \, \sin(2t).
\]

\noindent \textbf{Tracking Error:} at the $i^{\text{th}}$ iteration is computed as the Euclidean norm of the position errors 
$
|| \boldsymbol{\mathsf{e}}_i||_2 = 
\left\Vert\begin{matrix} y_{\mathrm{ref}} - y_\text{ee}^{(i)} \\ z_{\mathrm{ref}} - z_\text{ee}^{(i)} \end{matrix}\right\Vert_2.
$
 The learning gain was set to $L=0.05$, and the prediction gain to 
$K=1.5$, for all iterations, providing stable and consistent error reduction across iterations.

\end{document}